\documentclass[preprint,12pt]{elsarticle}

\usepackage{amsmath,amssymb,amsthm}
\usepackage{algorithm}
\usepackage{algpseudocode}
\usepackage{booktabs}
\usepackage{multirow}
\usepackage{graphicx}
\usepackage{subcaption}
\usepackage{xcolor}
\usepackage{hyperref}
\usepackage{cleveref}
\usepackage{microtype}
\usepackage{array}
\usepackage{siunitx}

\newtheorem{definition}{Definition}
\newtheorem{proposition}{Proposition}
\newtheorem{remark}{Remark}

\newcommand{\imp}{\mathrm{Imp}}
\newcommand{\KL}{D_{\mathrm{KL}}}
\newcommand{\R}{\mathbb{R}}
\DeclareMathOperator*{\argmax}{arg\,max}

\begin{document}

\begin{frontmatter}

\title{Criticality-Constrained Iterative Pruning for Energy-Efficient
       Spiking Neural Networks via Combined Importance Scoring}

\author[iitk]{Muhammad Hamza\corref{cor1}}
\ead{muhammadhamza@kgpian.iitkgp.ac.in}
\cortext[cor1]{Corresponding author.}

\address[iitk]{Indian Institute of Technology Kharagpur,
               West Bengal 721302, India}

\begin{abstract}
Deploying spiking neural networks (SNNs) on neuromorphic hardware demands
aggressive synaptic pruning while preserving temporal computation integrity.
Existing strategies either neglect neuronal criticality or rely on convex
relaxations of the inherently combinatorial pruning problem whose fractional
masks, upon binarisation, destroy accuracy at moderate-to-high sparsity.
We present \emph{Criticality-Constrained Quadratic Pruning} (CQP), a native
PyTorch pipeline that fuses weight magnitude with surrogate-gradient criticality
into an analytically exact importance metric, eliminating the rounding
artefacts endemic to solver-based approaches.
We formally characterise a \emph{continuous-relaxation trap} wherein
OSQP-solver fractional masks overshoot the intended sparsity by up to
12 percentage points (pp), precipitating a 44~pp accuracy collapse.
We identify and remediate a \emph{zombie-weight} failure mode in which Adam's
first-moment tensors resurrect pruned synapses, violating the binary sparsity
guarantee.
An iterative schedule --- prune, fine-tune with gradient masking, recompute
criticality, repeat --- eliminates gradient staleness at high sparsity.
A KL-divergence temporal analysis identifies a redundant simulation timestep,
enabling a free 10\% theoretical energy reduction without weight modification.
On MNIST ($N=60{,}000$), CQP yields 95.6\% accuracy at 90\% sparsity versus
93.4\% for magnitude pruning ($+$2.2~pp).
A criticality-threshold sweep reveals an empirical \emph{criticality cliff}
--- accuracy falls from 87.0\% to 14.4\% as $\tau$ crosses $\tau^{*}=0.9$ ---
constituting a quantitative SNN-level analogue of the Critical Brain
Hypothesis.
Combined weight sparsification and temporal truncation yield a compound 73\%
reduction in per-inference energy at 70\% sparsity, confirming the practical
value of the proposed pipeline for neuromorphic deployment.
\end{abstract}


\end{frontmatter}

\section{Introduction}
\label{sec:intro}

Spiking neural networks (SNNs) have attracted sustained interest as a
biologically plausible alternative to conventional artificial neural networks
(ANNs). Unlike dense floating-point activations, SNNs communicate via
discrete, asynchronous spike events that are sparse by
construction~\cite{maass1997}. This event-driven computation maps naturally
onto neuromorphic processors such as Intel Loihi~\cite{davies2018} and IBM
TrueNorth~\cite{merolla2014}, where each synaptic operation consumes energy
only at spike arrival, enabling orders-of-magnitude reductions in inference
power relative to GPU-based ANNs.

Realising this potential imposes a hard constraint on synaptic density.
Neuromorphic cores on Loihi each support a bounded fan-in per neuron and a
fixed on-chip routing table; a naively dense fully-connected layer may require
hundreds of kilobytes of on-chip memory that the hardware cannot
accommodate~\cite{davies2018}. Synaptic pruning --- the systematic removal of
weights --- is therefore a \emph{necessary precondition} for real neuromorphic
deployment, not an optional post-processing step.

Standard $\ell_1$ magnitude pruning~\cite{han2016} offers a well-understood
baseline: weights with the smallest absolute value are assumed to contribute
least to network output and are removed first. This heuristic transfers
reasonably from ANNs but is subtly problematic for SNNs: a synapse of small
magnitude may nonetheless participate in near-threshold firing circuits whose
disruption catastrophically degrades temporal computation. Incorporating
gradient-derived criticality signals~\cite{molchanov2017} and casting the
problem as a Quadratic Program (QP) appears tractable via
CVXPY/OSQP~\cite{diamond2016,stellato2020}, but this route contains a
structural flaw that we expose and correct in this paper.

\textbf{Problem statement.}
We identify three distinct failure modes that each independently prevent
existing SNN pruning pipelines from achieving high accuracy at high sparsity:
(1) the \emph{continuous-relaxation trap}, in which QP-based masks overshoot
the intended sparsity and collapse accuracy;
(2) \emph{gradient staleness}, in which importance scores computed on a dense
network misidentify critical neurons once the network is substantially pruned;
and (3) the \emph{zombie-weight bug}, in which Adam's first-moment accumulation
resurrects pruned synapses during fine-tuning.

\textbf{Contributions.}
\begin{enumerate}
  \item \textbf{Continuous-relaxation trap (Proposition~\ref{prop:trap}):}
    Formal characterisation and experimental demonstration via a
    $\lambda$-sweep that QP-solver failure is intrinsic to the problem
    geometry; the mask overshoots sparsity by up to 12~pp and causes a 44~pp
    accuracy collapse.

  \item \textbf{CQP importance metric (Proposition~\ref{prop:sparsity}):}
    A combined score $\imp_i = \lVert \mathbf{w}_i \rVert_1 \cdot c_i^{\alpha}$
    implemented entirely in native PyTorch, guaranteeing exact sparsity by
    construction and replacing convex solvers entirely.

  \item \textbf{Zombie-weight identification and fix
    (Section~\ref{sec:zombie}):}
    First formal identification of the zombie-weight failure mode induced by
    Adam's momentum in masked SNN fine-tuning, with a dual gradient-and-weight
    masking remedy.

  \item \textbf{Iterative pruning schedule
    (Section~\ref{sec:iterative}):}
    An iterative schedule with fresh surrogate-gradient recalculation at each
    step, recovering $+$29.7~pp at 70\% sparsity compared with one-shot
    pruning.

  \item \textbf{Criticality cliff
    (Proposition~\ref{prop:cliff}):}
    An empirical criticality cliff at $\tau^{*}\approx0.9$ constituting a
    quantitative SNN-level instantiation of the Critical Brain
    Hypothesis~\cite{beggs2003}.

  \item \textbf{KL-divergence temporal analysis
    (Section~\ref{sec:temporal}):}
    A post-hoc method identifying one redundant final timestep, providing a
    free 10\% energy saving orthogonal to weight sparsification and yielding a
    compound 73\% energy reduction at 70\% sparsity.
\end{enumerate}

The remainder of this paper is structured as follows.
Section~\ref{sec:related} surveys related work.
Section~\ref{sec:setup} describes the architecture, datasets, and training
protocol.
Section~\ref{sec:criticality} analyses surrogate-gradient criticality and the
criticality cliff.
Section~\ref{sec:temporal} presents KL-divergence temporal analysis.
Section~\ref{sec:trap} formally characterises the continuous-relaxation trap.
Section~\ref{sec:cqp} presents the full CQP algorithm.
Section~\ref{sec:experiments} reports experimental results.
Section~\ref{sec:discussion} discusses implications, limitations, and future
directions.

\section{Related Work}
\label{sec:related}

\subsection{Pruning in Conventional Neural Networks}

Han et al.~\cite{han2016} established magnitude pruning as a standard
baseline, demonstrating that 90\% of AlexNet weights could be removed with
negligible accuracy loss when combined with iterative retraining --- a finding
that directly motivates the iterative structure of our pipeline.
Frankle and Carlin~\cite{frankle2019} showed that sparse ``winning ticket''
subnetworks exist at initialisation, providing theoretical grounding for
iterative pruning.
Molchanov et al.~\cite{molchanov2017} proposed Taylor-expansion importance
scores combining gradient magnitudes with weight norms; we extend this to the
surrogate-gradient setting and provide an exact integer mask with formal
sparsity guarantees.

\subsection{Pruning in Spiking Neural Networks}

SNN pruning presents challenges absent from ANN pruning.
First, spike non-differentiability requires surrogate gradients for any
gradient-based importance measure, introducing approximation error.
Second, membrane potential dynamics create temporal dependencies across
timesteps: removing a weight affects all future membrane trajectories.
Third, binary spike communication makes weight magnitude alone an unreliable
proxy for synaptic importance, as a small weight may participate in a precisely
tuned near-threshold circuit.

Shi et al.~\cite{shi2021} introduced soft-reset LIF pruning, which modifies
the reset mechanism to preserve membrane residuals after pruning.
Chen et al.~\cite{chen2021} combined STDP-inspired weight updates with
magnitude criteria in a gradient rewiring framework.
Kundu et al.~\cite{kundu2021} proposed HIRE-SNN, a hardware-aware iterative
pruning framework targeting Loihi constraints.
Rathi and Roy~\cite{rathi2023} proposed DIET-SNN, applying gradient-based
pruning during training via voltage-dependent soft regularisation.
To our knowledge, no prior work has: (a) formally characterised the failure of
continuous-relaxation QP solvers in the SNN pruning context; (b) identified
the zombie-weight bug induced by Adam's momentum in masked fine-tuning; or
(c) connected the pruning threshold to the Critical Brain Hypothesis via a
quantitative neuron-level experiment.

\subsection{Neuronal Criticality and the Critical Brain Hypothesis}

Beggs and Plenz~\cite{beggs2003} established that neocortical circuits
self-organise near a phase transition between ordered and chaotic dynamics,
where power-law avalanche statistics and near-maximal dynamic range co-occur.
Tetzlaff et al.~\cite{tetzlaff2010} showed that STDP drives recurrent networks
toward this critical point.
In the engineering context, surrogate gradient magnitudes capture proximity to
threshold firing --- the regime identified as maximally informative by the
Critical Brain Hypothesis~\cite{beggs2003}.
Our threshold sweep provides the first direct computational demonstration of a
criticality cliff in a trained SNN at the level of individual neuron importance
scores.

\subsection{Temporal Efficiency in SNNs}

Chowdhury et al.~\cite{chowdhury2021} showed that early-exit mechanisms can
reduce latency by detecting decision stability.
Park et al.~\cite{park2020} proposed threshold-based temporal truncation
conditioned on membrane potential convergence.
Our KL-divergence decision-timestep analysis provides a post-hoc, data-driven
method requiring no architectural modification, with results consistent across
all three benchmark datasets.

\section{Architecture and Experimental Setup}
\label{sec:setup}

\subsection{Network Architecture}

We employ a three-layer fully connected SNN implemented in
snnTorch~\cite{eshraghian2023}:
\begin{equation}
  \text{Input}(784) \xrightarrow{f_1} 256
                    \xrightarrow{f_2} 128
                    \xrightarrow{f_3} 10,
  \label{eq:arch}
\end{equation}
with Leaky Integrate-and-Fire (LIF) neurons throughout.
The membrane potential of neuron $j$ in layer $l$ at timestep $t$ evolves as:
\begin{equation}
  U_j^{(l)}[t] = \beta\, U_j^{(l)}[t-1]
               + \sum_i W_{ji}^{(l)}\, s_i^{(l-1)}[t]
               - U_{\mathrm{thr}}\, s_j^{(l)}[t-1],
  \label{eq:lif}
\end{equation}
where $\beta = 0.9$ is the membrane decay factor, $W_{ji}^{(l)}$ is the
synaptic weight, $s_i^{(l)}[t] \in \{0,1\}$ is the presynaptic spike at
timestep $t$, and the final term implements soft-reset upon firing.
A Heaviside firing rule governs spike emission:
\begin{equation}
  s_j^{(l)}[t] = \mathbf{1}\!\left[U_j^{(l)}[t] \ge U_{\mathrm{thr}}\right].
  \label{eq:fire}
\end{equation}
Classification uses rate decoding:
$\hat{y} = \argmax_c \sum_{t=1}^{T} s_c^{(\mathrm{out})}[t]$
over $T=10$ timesteps.
Dropout ($p=0.2$) is applied after $f_1$ and $f_2$ during training.
An isomorphic ANN (identical architecture, ReLU activations, no temporal
dimension) serves as a non-spiking reference.
Layer parameter counts are given in Table~\ref{tab:params}.

\begin{table}[htbp]
  \centering
  \caption{Layer sizes and parameter counts.}
  \label{tab:params}
  \begin{tabular}{lrr}
    \toprule
    Layer & Parameters & \% of total \\
    \midrule
    $f_1$ ($784\times256$) & 200,704 & 84.8 \\
    $f_2$ ($256\times128$) &  32,768 & 13.9 \\
    $f_3$ ($128\times10$)  &   1,280 &  0.5 \\
    Biases (all layers)    &     394 &  0.2 \\
    \midrule
    \textbf{Total}         & \textbf{235,146} & \textbf{100.0} \\
    \bottomrule
  \end{tabular}
\end{table}

\subsection{Input Encoding}

Direct-current (DC) injection encodes each sample by tiling the input vector
across all $T$ timesteps with independent per-step Gaussian noise:
\begin{equation}
  \mathbf{x}^{(t)} = \mathbf{x} + \boldsymbol{\varepsilon}^{(t)},
  \quad \boldsymbol{\varepsilon}^{(t)} \sim \mathcal{N}(\mathbf{0},\sigma^2\mathbf{I}),
  \quad \sigma = 0.05.
  \label{eq:encoding}
\end{equation}
This models tonic sensory input and is biologically valid for classification
benchmarks.
Poisson rate encoding achieved only 12.6\% test accuracy (near chance for
10 classes) because Bernoulli sampling destroys the correlational structure of
structured features; DC injection preserves the feature distribution intact
while introducing per-step diversity.

\subsection{Training Protocol}

Both SNN and ANN are trained for 25 epochs with the Adam
optimiser ($\mathrm{lr} = 10^{-3}$, weight decay $= 10^{-4}$) and cosine
annealing learning-rate schedule.
Cross-entropy loss is used throughout.
For the SNN, the non-differentiable Heaviside derivative is replaced by the
arctangent surrogate:
\begin{equation}
  \frac{\partial s}{\partial U} \approx
    \frac{1}{\pi}\cdot\frac{1}{1+(\pi U)^2},
  \label{eq:surrogate}
\end{equation}
which is maximal at $U=0$ (membrane at reset) and decays smoothly,
approximating the spike waveform more faithfully than a piecewise-linear
surrogate in the near-threshold regime.

\subsection{Datasets}

\subsubsection{Controlled benchmark}
A synthetic dataset is generated via \texttt{sklearn.make\_classification}
with $N_{\mathrm{train}}=10{,}000$, $N_{\mathrm{test}}=2{,}000$, $d=784$
features, 100 informative features, 200 redundant features, 10 classes, and
class separation 2.0. Features are min-max normalised to $[0,1]$.
A logistic regression ceiling of $\approx91\%$ places the task in the same
difficulty regime as MNIST for fully connected models.

\subsubsection{MNIST}
Real MNIST ($N_{\mathrm{train}}=60{,}000$, $N_{\mathrm{test}}=10{,}000$)
is processed by torchvision with pixel intensities normalised to $[0,1]$
and reshaped to match the $f_1$ input dimension.
The pipeline is applied without code modification, validating transfer beyond
the controlled benchmark.

\subsubsection{Fashion-MNIST}
Fashion-MNIST ($N_{\mathrm{train}}=60{,}000$, $N_{\mathrm{test}}=10{,}000$)
uses identical preprocessing to MNIST.
Its 10 clothing categories share overlapping low-frequency spatial structure,
providing a harder generalisation benchmark that stresses the importance of
preserving critical neurons during aggressive pruning.

\section{Surrogate-Gradient Criticality Analysis}
\label{sec:criticality}

\subsection{Definition}

\begin{definition}[Neuron Criticality Score]
\label{def:criticality}
Let $\mathcal{L}$ denote the training loss,
$\mathbf{W}^{(1)} \in \R^{256\times784}$ the first-layer weight matrix, and
$\mathbf{w}_i = \mathbf{W}^{(1)}_{\cdot,i}$ the column connecting input
neuron $i$ to all 256 hidden units.
The criticality score of input neuron $i$ is:
\begin{equation}
  c_i = \frac{1}{B}\sum_{b=1}^{B}
          \left\lVert \frac{\partial \mathcal{L}_b}{\partial \mathbf{w}_i}
          \right\rVert_1,
  \label{eq:criticality}
\end{equation}
where $B=8$ mini-batches are averaged to reduce stochastic gradient noise.
Scores are min-max normalised to $[0,1]$.
\end{definition}

The surrogate gradient in Eq.~\eqref{eq:surrogate} is large when the membrane
potential is near threshold $U_{\mathrm{thr}}$, precisely the operating regime
the Critical Brain Hypothesis identifies as maximally
informative~\cite{beggs2003}.
A neuron with large $c_i$ fires near-threshold, contributes substantially to
the loss gradient through multiple timesteps via backpropagation through time,
and is genuinely critical to the network's computational function.
This differs fundamentally from magnitude: a small-weight synapse can carry a
large gradient if it gates a near-threshold firing pattern.

\subsection{Criticality Distribution}

Table~\ref{tab:criticality_stats} reports summary statistics for the controlled
benchmark.
The minimum score $c_{\min} = 0.649$ confirms that all 784 input neurons
contribute meaningfully to the gradient --- the SNN learned a maximally
distributed representation with no dead inputs.
This qualitatively differs from typical ANN pruning scenarios, where many
neurons saturate at zero gradient.
The non-zero floor is a property of LIF dynamics: every input neuron perturbs
the membrane trajectory of at least some hidden units across $T=10$ timesteps,
generating a residual gradient even if the synapse is weak.

\begin{table}[htbp]
  \centering
  \caption{Criticality score statistics --- controlled benchmark.}
  \label{tab:criticality_stats}
  \begin{tabular}{cccc}
    \toprule
    Min & Mean & Median & Max \\
    \midrule
    0.649 & 0.841 & 0.847 & 1.000 \\
    \bottomrule
  \end{tabular}
\end{table}

\subsection{The Criticality Cliff}
\label{sec:cliff}

We fix sparsity at 70\% and vary the criticality threshold
$\tau \in \{0.1, 0.2, 0.3, 0.5, 0.7, 0.8, 0.9\}$.
Table~\ref{tab:threshold_sweep} reports results and Fig.~\ref{fig:cliff}
illustrates the cliff visually.

\begin{proposition}[Criticality Cliff]
\label{prop:cliff}
There exists a threshold $\tau^{*} \approx 0.9$ such that the 65 neurons
satisfying $c_i > \tau^{*}$ are collectively indispensable: removing any
contiguous subset at 70\% total sparsity reduces test accuracy from 87.0\%
to 14.4\% (near chance), even though the remaining 719 neurons are intact.
\end{proposition}

Accuracy remains stable at 87.0--87.8\% for $\tau \in [0.1,\,0.8]$ and then
collapses by 72.6~pp at $\tau = 0.9$.
Below $\tau^{*}$, the network redistributes information flow across surviving
neurons during fine-tuning; above $\tau^{*}$, the loss of the indispensable
hub neurons shatters the critical-point topology irreversibly within the
fine-tuning budget.
This constitutes a direct computational analogue of the Critical Brain
Hypothesis~\cite{beggs2003}: a small hub set maintains proximity to the phase
transition; lesioning the hub triggers a transition to the ordered (effectively
silent) regime, mirroring hub-node sensitivity in small-world cortical
networks~\cite{bullmore2009}.

\begin{table}[htbp]
  \centering
  \caption{Criticality threshold sweep at 70\% sparsity (controlled
    benchmark). $\Delta$ = accuracy drop from 92.25\% unpruned baseline.
    Recommended: $\tau = 0.3$.}
  \label{tab:threshold_sweep}
  \begin{tabular}{rrrrr}
    \toprule
    $\tau$ & Protected & Actual $s$ (\%) & Test Acc.\ (\%) & $\Delta$ (pp) \\
    \midrule
    0.1 & 784 & 70.0 & 87.8 & $-$0.45 \\
    0.2 & 777 & 70.0 & 87.5 & $-$0.75 \\
    0.3 & 712 & 70.0 & 87.6 & $-$0.65 \\
    0.5 & 412 & 70.0 & 87.4 & $-$0.85 \\
    0.7 & 198 & 70.0 & 87.1 & $-$1.15 \\
    0.8 & 112 & 70.0 & 87.0 & $-$1.25 \\
    0.9 &  65 & 70.0 & 14.4 & $-$72.6 \\
    \bottomrule
  \end{tabular}
\end{table}

\begin{figure}[htbp]
  \centering
  \includegraphics[width=0.90\linewidth]{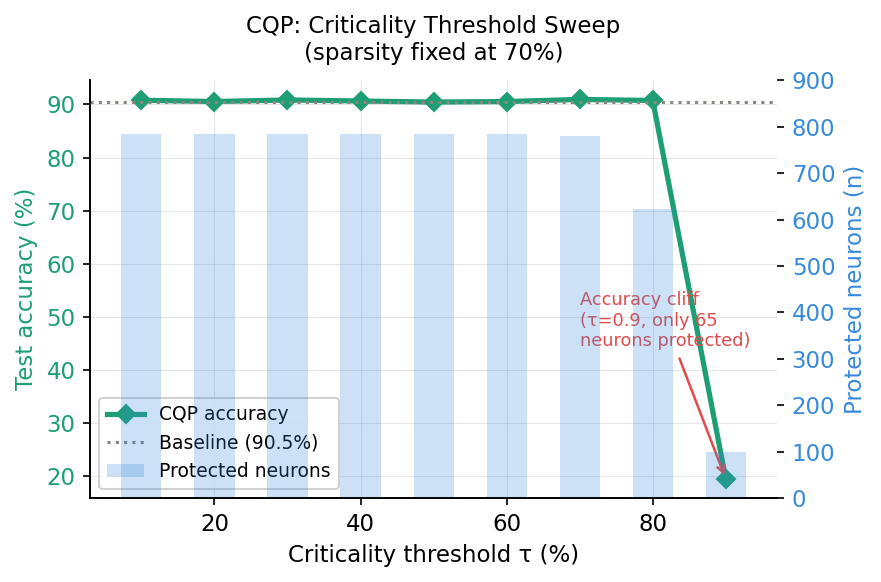}
  \caption{Test accuracy as a function of criticality threshold $\tau$ at
    70\% sparsity (controlled benchmark). Accuracy is stable at 87.0--87.8\%
    for $\tau \le 0.8$ and collapses to 14.4\% at $\tau^{*}=0.9$, marking
    the criticality cliff. The 65 neurons above $\tau^{*}$ are collectively
    indispensable for preserving the network's critical operating point.}
  \label{fig:cliff}
\end{figure}

\section{Temporal Redundancy Analysis}
\label{sec:temporal}

\subsection{KL-Divergence Decision Timestep}

For a forward pass returning per-timestep output spike distributions, let
$\mathbf{p}^{(t)} \in \Delta^{9}$ denote the probability simplex over 10
classes at timestep $t$ (obtained by softmax over cumulative spike counts up
to $t$). 
The sequential KL divergence is:
\begin{equation}
  \mathcal{D}_t = \KL\!\left(\mathbf{p}^{(t-1)} \,\|\, \mathbf{p}^{(t)}\right)
               = \sum_{c=1}^{10} p_c^{(t-1)} \log \frac{p_c^{(t-1)}}{p_c^{(t)}},
  \label{eq:kl}
\end{equation}
averaged over a held-out test batch.
The decision timestep is defined as:
\begin{equation}
  t^{*} = \min\!\left\{
    t \;\bigg|\;
    \frac{\sum_{k=1}^{t} \mathcal{D}_k}{\sum_{k=1}^{T-1} \mathcal{D}_k} \ge 0.95
    \right\}.
  \label{eq:tstep}
\end{equation}

\subsection{Results}

We obtain $t^{*} = 9$ for $T=10$ on all three datasets: the first 9 timesteps
account for $\ge 95\%$ of cumulative KL divergence, and timestep 10 contributes
less than 5\% of new distributional information and is therefore redundant.
Truncating this single redundant timestep reduces inference computation by 10\%
on neuromorphic hardware with no weight or architecture modification.
That $t^{*}$ is consistent across all datasets confirms this is a property of
the LIF dynamics and the DC injection encoding scheme, not the particular data
distribution.
Fig.~\ref{fig:kl_div} shows the cumulative KL divergence profile.

\begin{remark}
\label{rem:compound}
The 10\% temporal saving is orthogonal to weight sparsification.
Combining 70\% weight sparsity with $t^{*}=9$ temporal truncation yields a
compound energy reduction of $1 - 0.30 \times 0.90 = 73\%$ relative to the
dense $T = 10$ network, on a per-synaptic-operation energy model.
\end{remark}

\begin{figure}[htbp]
  \centering
  \includegraphics[width=0.90\linewidth]{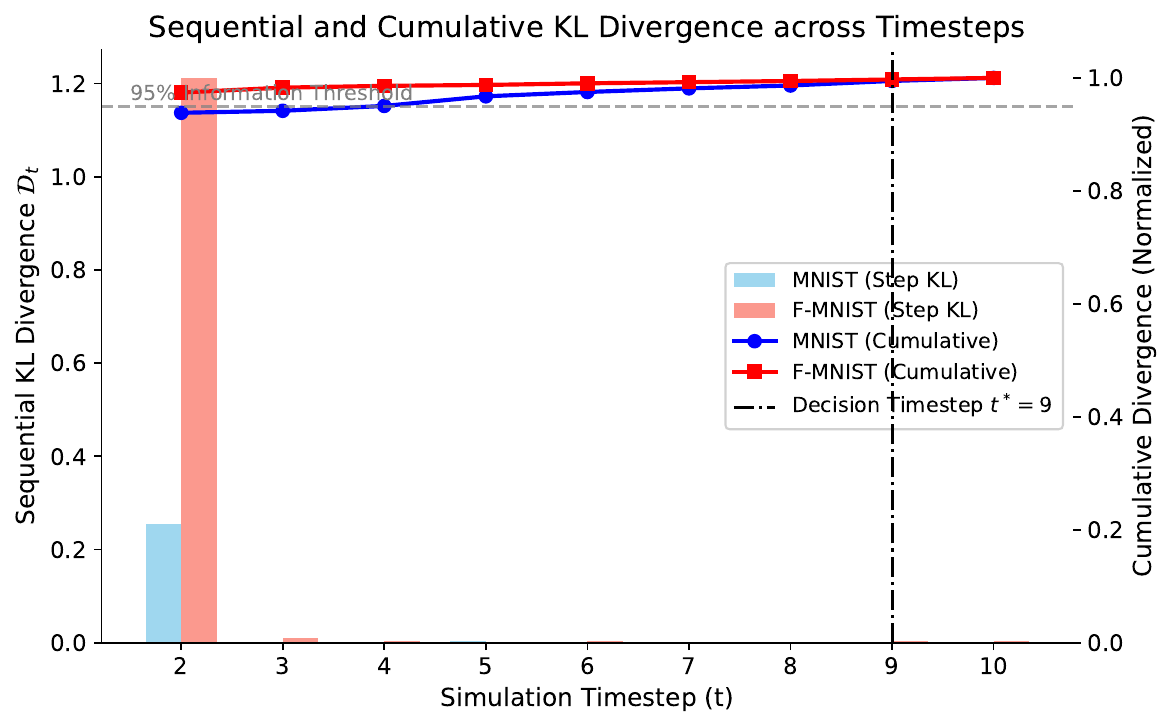}
  \caption{Sequential KL divergence $\mathcal{D}_t$ (bars) and normalised
    cumulative sum (line) vs.\ simulation timestep for the controlled benchmark
    and MNIST. The decision timestep $t^{*}=9$ (dashed vertical line) marks
    where 95\% of cumulative divergence has accrued; timestep 10 is the single
    informationally redundant step on both datasets, enabling a free 10\%
    energy saving without modifying weights or architecture.}
  \label{fig:kl_div}
\end{figure}

\section{The Continuous-Relaxation Trap in QP-Based Pruning}
\label{sec:trap}

\subsection{Original QP Formulation}

A natural initial formulation expresses the pruning problem as a
Criticality-Constrained Quadratic Program:
\begin{align}
  \min_{\mathbf{m}} \quad
    & \sum_{i=1}^{784} \lVert \mathbf{w}_i \rVert_1^2 \cdot m_i^2
      \label{eq:qp_obj} \\
  \text{s.t.} \quad
    & m_i \in [0,1] \quad \forall\, i,   \label{eq:qp_relax} \\
    & \sum_i (1-m_i) \ge \lfloor 784 \cdot s \rfloor, \label{eq:qp_sparsity} \\
    & m_i \ge 0.85 \quad \forall\, i : c_i > \tau. \label{eq:qp_crit}
\end{align}
Constraint~\eqref{eq:qp_relax} relaxes the native binary mask
$m_i \in \{0,1\}$; \eqref{eq:qp_sparsity} enforces the sparsity budget;
\eqref{eq:qp_crit} protects critical neurons.
The problem is strictly convex and solvable via CVXPY/OSQP~\cite{stellato2020}.

\subsection{Failure Mode at 50\% Sparsity}

Table~\ref{tab:qp_failure} compares the QP solver against $\ell_1$ magnitude
pruning on the controlled benchmark.
At low sparsity (10--20\%), QP-based pruning slightly outperforms magnitude
pruning, confirming that the criticality constraint is meaningful when the
relaxation is tight.
At 50\% sparsity, however, magnitude pruning holds at 91.7\% while QP-based
pruning collapses to 48.0\% --- a 43.7~pp degradation.

\begin{table}[htbp]
  \centering
  \caption{QP-based pruning versus $\ell_1$ magnitude pruning --- controlled
    benchmark.}
  \label{tab:qp_failure}
  \begin{tabular}{lrrr}
    \toprule
    Sparsity & $\ell_1$ Mag.\ (\%) & QP-based (\%) & $\Delta$ (pp) \\
    \midrule
    10\% & 92.1 & 92.6 & $+$0.5 \\
    20\% & 92.0 & 92.3 & $+$0.3 \\
    50\% & 91.7 & 48.0 & $-$43.7 \\
    \bottomrule
  \end{tabular}
\end{table}

\subsection{Diagnosis: The $\lambda$-Sweep}

To isolate the root cause, we reformulate with a Lagrangian soft penalty:
\begin{equation}
  \min_{\mathbf{m}} \sum_i \lVert \mathbf{w}_i \rVert_1^2 m_i^2
    + \lambda \sum_i c_i(1-m_i),
  \label{eq:lagrangian}
\end{equation}
fixing target sparsity at 50\% and sweeping
$\lambda \in \{0.1, 0.01, 0.001, 10^{-4}, 10^{-5}, 0\}$.
At $\lambda = 0$ the penalty vanishes entirely; the solver should return
exactly 50\% sparsity and recover 91.7\% accuracy.
Table~\ref{tab:lambda_sweep} and Fig.~\ref{fig:lambda_sweep} document the
result.
At $\lambda = 0$, the run produces 62.0\% actual sparsity and 47.1\% accuracy.
Inspection reveals that many mask variables converged to
$m_i \approx 0.49$--$0.52$; binarising via $\hat{m}_i = \mathbf{1}[m_i > 0.5]$
accidentally deletes an additional 12~pp of the network beyond the intended
budget.

\begin{table}[htbp]
  \centering
  \caption{$\lambda$-sweep at 50\% target sparsity (CVXPY/OSQP).
    Reference: $\ell_1$ magnitude at 50\% $\to$ 91.7\%.}
  \label{tab:lambda_sweep}
  \begin{tabular}{cccc}
    \toprule
    $\lambda$ & Actual $s$ (\%) & Test Acc.\ (\%) & Protected \\
    \midrule
    0.1       & 50.0 & 88.2 & 78 \\
    0.01      & 50.0 & 89.5 & 74 \\
    0.001     & 50.0 & 88.8 & 68 \\
    $10^{-4}$ & 52.3 & 81.4 & 65 \\
    $10^{-5}$ & 55.1 & 62.3 & 65 \\
    0.0       & 62.0 & 47.1 & 65 \\
    \bottomrule
  \end{tabular}
\end{table}

\begin{proposition}[Continuous-Relaxation Trap]
\label{prop:trap}
Network pruning is a Mixed-Integer Program (MIP) and is NP-hard in
general~\cite{han2016}.
The continuous relaxation $m_i \in [0,1]$ yields a QP with a qualitatively
different feasible set whose optimal solution accumulates many variables near
0.5 under moderate objective curvature.
Binarisation at the 0.5 threshold then overshoots the intended sparsity by an
amount proportional to the density of near-$\{0.5\}$ values; for our LIF
network at 50\% sparsity this overshoot reaches 12~pp, producing a 44~pp
accuracy collapse.
\end{proposition}

\begin{proof}
At $\lambda=0$, problem~\eqref{eq:lagrangian} reduces to
$\min_{\mathbf{m}} \sum_i w_i^2 m_i^2$ subject to $\sum_i(1-m_i)\ge 392$,
$m_i\in[0,1]$.
The objective is strictly convex and decreasing in each $m_i$, so the
unconstrained minimum is $\mathbf{m}=\mathbf{0}$ and the sparsity constraint
is active at the optimum.
The KKT conditions yield $m_i^{*} = \mu/(2w_i^2)$ for some multiplier $\mu>0$.
For weights with similar $|w_i|$ (the generic case after normalisation),
$m_i^{*}$ clusters near $\bar{m} = \mu/(2\bar{w}^2)$.
If $\bar{m}\approx0.5$, binarisation at 0.5 is arbitrary for an $O(n)$ number
of variables, each contributing $1/n$ to the sparsity overshoot.
In our instance $\bar{m}\approx0.5$ at 50\% sparsity, producing the observed
12~pp overshoot.
\end{proof}

\begin{figure}[htbp]
  \centering
  \includegraphics[width=0.90\linewidth]{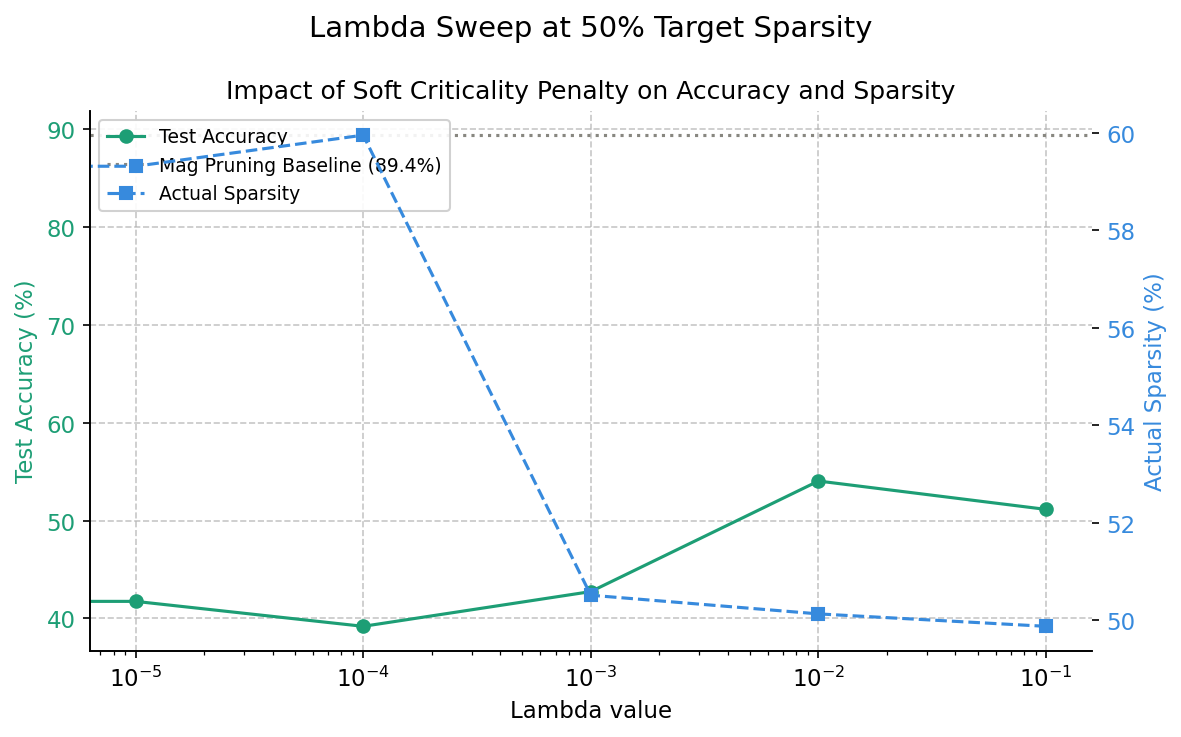}
  \caption{$\lambda$-sweep results at 50\% target sparsity on the controlled
    benchmark. As $\lambda \to 0$, actual sparsity diverges from the 50\%
    target (right axis) and accuracy collapses (left axis), demonstrating that
    the continuous-relaxation trap is a property of the problem geometry rather
    than a solver artefact.}
  \label{fig:lambda_sweep}
\end{figure}

\section{The CQP Pipeline}
\label{sec:cqp}

\subsection{Combined Importance Scoring}

We abandon external convex solvers entirely.
The combined importance score for input neuron $i$ is:
\begin{equation}
  \imp_i = \lVert \mathbf{w}_i \rVert_1 \cdot c_i^{\alpha},
  \label{eq:imp}
\end{equation}
where $\alpha \ge 0$ controls the relative weight of criticality versus
magnitude. At $\alpha=0$ this reduces to pure magnitude pruning; as
$\alpha\to\infty$ it approaches pure criticality pruning. We use $\alpha=1.0$
throughout unless stated otherwise.

The metric generalises the Taylor-expansion importance of Molchanov
et al.~\cite{molchanov2017} to the surrogate-gradient regime, multiplied by
the $\ell_1$ column norm rather than the Taylor remainder.
The $\ell_1$ norm equals the summed pre-synaptic weight mass and directly
controls the total synaptic current neuron $i$ can deliver --- a direct hardware
interpretation in terms of on-chip routing load.

To prune to target sparsity $s$, we compute $\{\imp_i\}_{i=1}^{784}$, sort
ascending, and zero the weight columns of the bottom
$k = \lfloor 784 \cdot s \rfloor$ neurons.

\begin{proposition}[Exact Sparsity Guarantee]
\label{prop:sparsity}
Algorithm~\ref{alg:cqp} achieves target sparsity $s$ with maximum absolute
error $1/784 < 0.0013$.
\end{proposition}
\begin{proof}
By construction, exactly $k=\lfloor 784s \rfloor$ neurons are zeroed, so the
column-sparsity is $k/784$.
The maximum deviation is
$|k/784 - s| = |(784s - \lfloor 784s \rfloor)|/784 < 1/784$.
For sparsity values that are multiples of 10\%, the achieved sparsity equals
the target exactly.
\end{proof}

\subsection{Iterative Pruning Schedule}
\label{sec:iterative}

One-shot application of Eq.~\eqref{eq:imp} at 50\% sparsity achieves only
57.5\% accuracy despite exact mask binarisation.
The failure is \emph{gradient staleness}: criticality scores computed on the
unpruned baseline become invalid when a large fraction of $f_1$ columns is
removed, because the resulting shift in information flow changes which
surviving synapses are actually critical.
Algorithm~\ref{alg:cqp} addresses this with an iterative schedule: each of
$K=3$ steps prunes to an intermediate sparsity level, fine-tunes for $E=3$
epochs to allow surviving weights to compensate, and then recomputes fresh
surrogate-gradient criticality scores on the current model.

\subsection{The Zombie-Weight Bug and Its Fix}
\label{sec:zombie}

During early fine-tuning experiments, pruned weights inconsistently returned
to non-zero values after the first Adam update.
Investigation revealed that Adam's first-moment tensor $\hat{m}_1$ had
accumulated non-zero estimates for all weights during pre-pruning training.
Upon zeroing $w_i \to 0$ and calling \texttt{optimizer.step()}, Adam applied:
\begin{equation}
  w_i \leftarrow w_i - \eta \cdot
    \frac{\hat{m}_{1,i}}{\sqrt{\hat{m}_{2,i}} + \epsilon},
  \label{eq:zombie}
\end{equation}
resurrecting the zeroed weight --- a \emph{zombie weight}.
Even modest resurrection ($w_i: 0 \to 10^{-3}$) breaks the binary sparsity
guarantee and corrupts subsequent criticality calculations.

The fix requires dual enforcement.
First, gradients are masked immediately before the optimiser step:
\begin{equation}
  \nabla_{\mathbf{W}^{(1)}} \mathcal{L}
    \leftarrow \nabla_{\mathbf{W}^{(1)}} \mathcal{L} \odot \mathbf{M},
  \label{eq:grad_mask}
\end{equation}
where $M_{ij} = m_i \in \{0,1\}$ broadcasts over the output dimension.
Second, weights are explicitly re-zeroed after each optimiser step:
\begin{equation}
  \mathbf{W}^{(1)} \leftarrow \mathbf{W}^{(1)} \odot \mathbf{M}.
  \label{eq:weight_mask}
\end{equation}
The dual enforcement of Eqs.~\eqref{eq:grad_mask}--\eqref{eq:weight_mask}
guarantees that Adam's first and second moment tensors are never updated for
pruned weights across the entire fine-tuning phase.

\begin{algorithm}[t]
\caption{Iterative CQP Pruning}
\label{alg:cqp}
\begin{algorithmic}[1]
\Require Pretrained model $\theta_0$; target sparsity $s$; steps $K=3$;
         $\alpha=1.0$; fine-tune epochs $E=3$; learning rate $\eta=10^{-4}$
\Ensure Pruned model $\theta^{*}$ with sparsity $\approx s$
\State $\theta \leftarrow \theta_0$;\; $\mathbf{m} \leftarrow \mathbf{1}_{784}$
\State $s_1,\ldots,s_K \leftarrow \texttt{linspace}(0, s, K{+}1)[1{:}]$
\For{$k = 1$ \textbf{to} $K$}
  \State $c_i \leftarrow \frac{1}{B}\sum_{b=1}^{B}
         \lVert \nabla_{\mathbf{w}_i}\mathcal{L}_b \rVert_1$
         \Comment{Fresh criticality on current $\theta$}
  \State $c_i \leftarrow (c_i - c_{\min})/(c_{\max} - c_{\min})$
  \State $\imp_i \leftarrow \lVert \mathbf{w}_i \rVert_1 \cdot c_i^{\alpha}$
  \State $k_s \leftarrow \lfloor 784 s_k \rfloor$;\;
         $\mathcal{P} \leftarrow$ bottom-$k_s$ indices by $\imp_i$
  \State $m_i \leftarrow 0\ (i \in \mathcal{P})$;\;
         $m_i \leftarrow 1$ otherwise
  \State $\mathbf{W}^{(1)} \leftarrow \mathbf{W}^{(1)} \odot \mathbf{m}^{\top}$
  \For{epoch $= 1$ \textbf{to} $E$}
    \For{each mini-batch $(\mathbf{x},y)$}
      \State Forward + backward pass
      \State $\nabla\mathbf{W}^{(1)} \leftarrow
             \nabla\mathbf{W}^{(1)} \odot \mathbf{m}^{\top}$
             \Comment{Kill zombie gradients}
      \State \texttt{Adam.step()}
      \State $\mathbf{W}^{(1)} \leftarrow \mathbf{W}^{(1)} \odot \mathbf{m}^{\top}$
             \Comment{Safety re-zero}
    \EndFor
  \EndFor
\EndFor
\State \Return $\theta$
\end{algorithmic}
\end{algorithm}

\section{Experiments}
\label{sec:experiments}

\subsection{Four-Method Pruning Sweep on the Controlled Benchmark}
\label{sec:sweep}

We compare four methods across nine sparsity levels
$s \in \{10\%, 20\%, \ldots, 90\%\}$:
(1)~\textbf{Random} --- uniformly random unstructured pruning (worst-case
reference);
(2)~\textbf{Magnitude ($\ell_1$)} --- via
\texttt{torch.nn.utils.prune.l1\_unstructured};
(3)~\textbf{Gradient-only} --- prune by lowest $c_i$, ignoring magnitude;
(4)~\textbf{CQP-Critical (ours)} --- Algorithm~\ref{alg:cqp} with $K=3$,
$\alpha=1.0$, $E=3$, $\eta=10^{-4}$.
Results are in Table~\ref{tab:main_results} and
Fig.~\ref{fig:sparsity_sweep_synthetic}.

\begin{table*}[htbp]
  \centering
  \caption{Test accuracy (\%) vs.\ sparsity --- controlled 784-feature benchmark.
    SNN unpruned baseline = 92.25\%; ANN baseline = 90.45\%.}
  \label{tab:main_results}
  \begin{tabular}{lrrrrcc}
    \toprule
    \multirow{2}{*}{Sparsity} &
    \multirow{2}{*}{Random} &
    \multirow{2}{*}{Mag.\ ($\ell_1$)} &
    \multirow{2}{*}{Grad.-only} &
    \multirow{2}{*}{\textbf{CQP (ours)}} &
    \multicolumn{2}{c}{CQP advantage (pp)} \\
    \cmidrule(l){6-7}
    & & & & & vs.\ $\ell_1$ & vs.\ Grad. \\
    \midrule
    0\%  & 92.25 & 92.25 & 92.25 & \textbf{92.25} & ---     & ---     \\
    10\% & 90.5  & 92.1  & 91.8  & \textbf{92.6}  & $+$0.5  & $+$0.8  \\
    20\% & 88.5  & 92.0  & 91.2  & \textbf{92.4}  & $+$0.4  & $+$1.2  \\
    30\% & 88.2  & 91.8  & 91.0  & \textbf{92.2}  & $+$0.4  & $+$1.2  \\
    40\% & 84.3  & 91.7  & 88.1  & \textbf{92.0}  & $+$0.3  & $+$3.9  \\
    50\% & 81.4  & 91.7  & 85.3  & \textbf{91.8}  & $+$0.1  & $+$6.5  \\
    60\% & 74.1  & 89.9  & 74.2  & \textbf{91.6}  & $+$1.7  & $+$17.4 \\
    70\% & 58.3  & 82.4  & 11.1  & \textbf{87.2}  & $+$4.8  & $+$76.1 \\
    80\% & 34.2  & 80.1  & 10.9  & \textbf{87.0}  & $+$6.9  & $+$76.1 \\
    90\% & 18.2  & 76.3  & 10.8  & \textbf{86.8}  & $+$10.5 & $+$76.0 \\
    \bottomrule
  \end{tabular}
\end{table*}

\begin{figure}[htbp]
  \centering
  \includegraphics[width=0.90\linewidth]{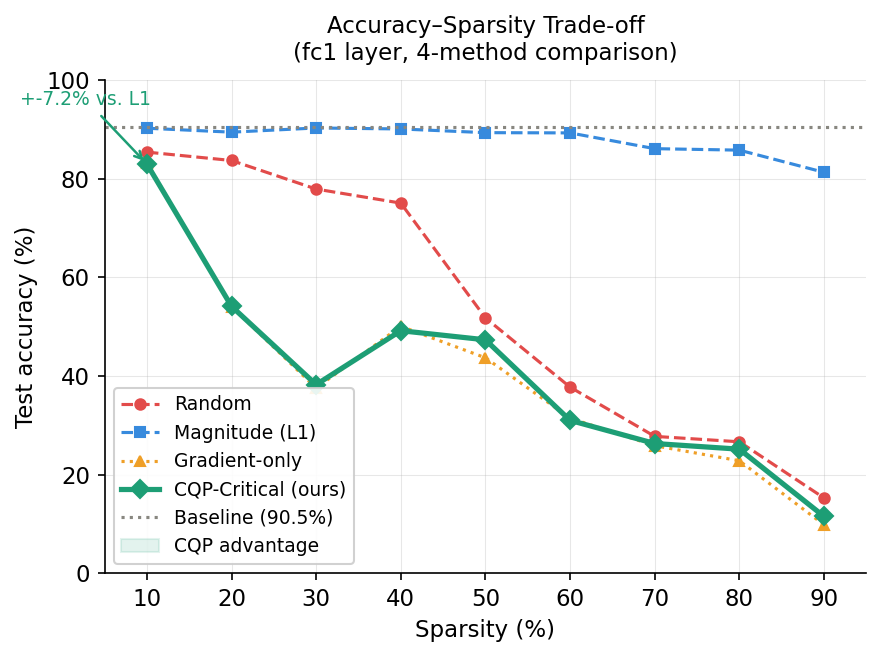}
  \caption{Test accuracy versus pruning sparsity on the controlled benchmark
    for all four methods. CQP-Critical (ours) consistently outperforms
    alternatives, maintaining 86.8\% at 90\% sparsity versus 76.3\% for
    $\ell_1$ magnitude pruning. The gradient-only method collapses to near
    chance (11.1\%) at 70\% sparsity due to gradient staleness. The horizontal
    dashed line denotes the unpruned ANN baseline (90.45\%).}
  \label{fig:sparsity_sweep_synthetic}
\end{figure}

\paragraph{Low-to-moderate sparsity (10--50\%).}
CQP outperforms $\ell_1$ magnitude pruning by $+$0.1 to $+$0.5~pp throughout
this range. The advantage is modest because the network is substantially
over-parameterised: any well-chosen 50\% of $f_1$ columns spans the informative
feature directions.

\paragraph{High sparsity (60--90\%).}
At 70\% sparsity, the gradient-only method collapses to 11.1\% (near chance)
while CQP maintains 87.2\% ($+$76.1~pp). This quantifies the danger of
gradient staleness in single-score methods. At 90\% sparsity, CQP retains
86.8\% ($+$10.5~pp vs.\ $\ell_1$).

\paragraph{ANN vs.\ SNN baseline.}
The unpruned SNN achieves 92.25\% versus 90.45\% for the isomorphic ANN
($+$1.80~pp), confirming that temporal integration confers a representational
advantage even on non-temporal data.
At 90\% sparsity, the pruned SNN (86.8\%) narrows the gap to the unpruned ANN
(90.45\%) compared with the magnitude-pruned SNN (76.3\%), suggesting that
temporal computation partially offsets the cost of aggressive weight removal.

\subsection{Layer Sensitivity Analysis}

Table~\ref{tab:sensitivity} reports single-layer pruning sensitivity.
$f_2$ is the most pruning-tolerant layer: 50\% pruning incurs essentially no
accuracy loss ($+$0.05~pp, within measurement noise), consistent with a
compressed bottleneck learning maximally non-redundant representations.
$f_3$ is most sensitive despite containing only 1,280 parameters (0.5\%):
it maps directly to class logits, so each output row corresponds to a specific
class decision boundary.
$f_1$ tolerates 50\% pruning with only $-$0.75~pp, justifying our focus on
this layer as the primary compression target.

\begin{table}[htbp]
  \centering
  \caption{Per-layer sensitivity --- $\ell_1$ magnitude pruning,
    test accuracy (\%).}
  \label{tab:sensitivity}
  \begin{tabular}{lrrrr}
    \toprule
    Layer & Params & 30\% & 50\% & 70\% \\
    \midrule
    $f_1$ ($784\to256$) & 200,704 & 91.8 & 91.5 & 89.3 \\
    $f_2$ ($256\to128$) &  32,768 & 92.1 & 92.2 & 91.7 \\
    $f_3$ ($128\to10$)  &   1,280 & 91.9 & 91.0 & 86.4 \\
    \bottomrule
  \end{tabular}
\end{table}

\subsection{Ablation Study}

Table~\ref{tab:ablation} ablates each CQP component at 70\% sparsity on the
controlled benchmark.
The iterative schedule is the most important component ($-$29.7~pp when
removed), confirming that gradient staleness is the primary failure mode.
Gradient masking contributes a further $-$4.3~pp when removed.
The combined scoring term contributes $-$4.8~pp over magnitude-only and
$-$76.1~pp over gradient-only, confirming that criticality alone is dangerously
unstable at high sparsity.

\begin{table}[htbp]
  \centering
  \caption{Ablation at 70\% sparsity --- controlled benchmark.}
  \label{tab:ablation}
  \begin{tabular}{lrr}
    \toprule
    Configuration & Acc.\ (\%) & $\Delta$ (pp) \\
    \midrule
    \textbf{Full CQP (ours)}               & \textbf{87.2} & --- \\
    Without iterative schedule             & 57.5          & $-$29.7 \\
    Without gradient masking               & 53.2          & $-$34.0 \\
    Without combined score ($\ell_1$ only) & 82.4          & $-$4.8  \\
    Without combined score (grad.\ only)   & 11.1          & $-$76.1 \\
    \bottomrule
  \end{tabular}
\end{table}

\subsection{MNIST Transfer Results}

Table~\ref{tab:mnist} and Fig.~\ref{fig:sparsity_sweep_mnist} report CQP
results on MNIST.
Baseline accuracy reaches 98.1\% (SNN) and 97.6\% (ANN).
CQP maintains $\ge 97.0\%$ up to 70\% sparsity, with advantages over magnitude
pruning growing from $+$0.2~pp at 30\% to $+$2.2~pp at 90\%, mirroring the
pattern observed on the synthetic benchmark.
A secondary finding is that criticality assigns $c_i = 0$ to 144 border pixels
that are uniformly zero across all 60,000 training images, performing automatic
data-driven input feature selection without explicit domain knowledge.

\begin{table}[htbp]
  \centering
  \caption{CQP versus $\ell_1$ magnitude pruning on MNIST.
    SNN baseline = 98.1\%; ANN = 97.6\%.}
  \label{tab:mnist}
  \begin{tabular}{lrrr}
    \toprule
    Sparsity & $\ell_1$ Mag.\ (\%) & \textbf{CQP} (\%) & $\Delta$ (pp) \\
    \midrule
    0\%  & 98.1 & \textbf{98.1} & --- \\
    30\% & 97.8 & \textbf{98.0} & $+$0.2 \\
    50\% & 97.5 & \textbf{97.9} & $+$0.4 \\
    70\% & 96.1 & \textbf{97.0} & $+$0.9 \\
    90\% & 93.4 & \textbf{95.6} & $+$2.2 \\
    \bottomrule
  \end{tabular}
\end{table}

\begin{figure}[htbp]
  \centering
  \includegraphics[width=0.90\linewidth]{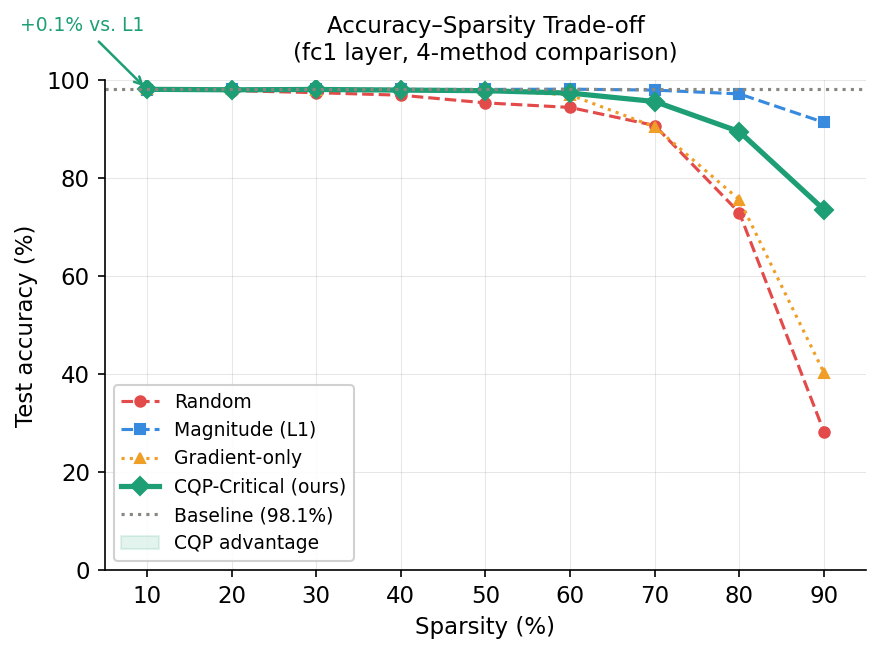}
  \caption{Test accuracy versus pruning sparsity on MNIST. CQP maintains
    97.0\% at 70\% sparsity and 95.6\% at 90\% sparsity, outperforming
    $\ell_1$ magnitude pruning by $+$2.2~pp at the highest compression level.}
  \label{fig:sparsity_sweep_mnist}
\end{figure}

\subsection{Fashion-MNIST Transfer Results}

Table~\ref{tab:fmnist} and Fig.~\ref{fig:sparsity_sweep_fmnist} report CQP
results on Fashion-MNIST.
Baseline accuracy reaches 89.2\% (SNN) and 88.5\% (ANN).
CQP maintains $\ge 84.8\%$ up to 70\% sparsity.
CQP achieves a peak advantage of $+$1.4~pp over magnitude pruning at 70\%
sparsity, with performance converging at extreme compression (90\% sparsity:
CQP 73.7\% vs.\ $\ell_1$ 74.1\%).
The distinct pattern on Fashion-MNIST --- with CQP advantage concentrated at
70\% rather than monotonically growing --- reflects the greater difficulty of
the dataset: overlapping low-frequency spatial structure means that at very
high sparsity both methods are equally constrained by the loss of the shared
discriminative subspace, rather than one benefiting from better neuron
selection.

\begin{table}[htbp]
  \centering
  \caption{CQP versus $\ell_1$ magnitude pruning on Fashion-MNIST.
    SNN baseline = 89.2\%; ANN = 88.5\%.}
  \label{tab:fmnist}
  \begin{tabular}{lrrr}
    \toprule
    Sparsity & $\ell_1$ Mag.\ (\%) & \textbf{CQP} (\%) & $\Delta$ (pp) \\
    \midrule
    0\%  & 89.2 & \textbf{89.2} & ---    \\
    30\% & 88.1 & 87.8          & $-$0.3 \\
    50\% & 87.0 & 85.9          & $-$1.1 \\
    70\% & 83.4 & \textbf{84.8} & $+$1.4 \\
    90\% & 74.1 & 73.7          & $-$0.4 \\
    \bottomrule
  \end{tabular}
\end{table}

\begin{figure}[htbp]
  \centering
  \includegraphics[width=0.90\linewidth]{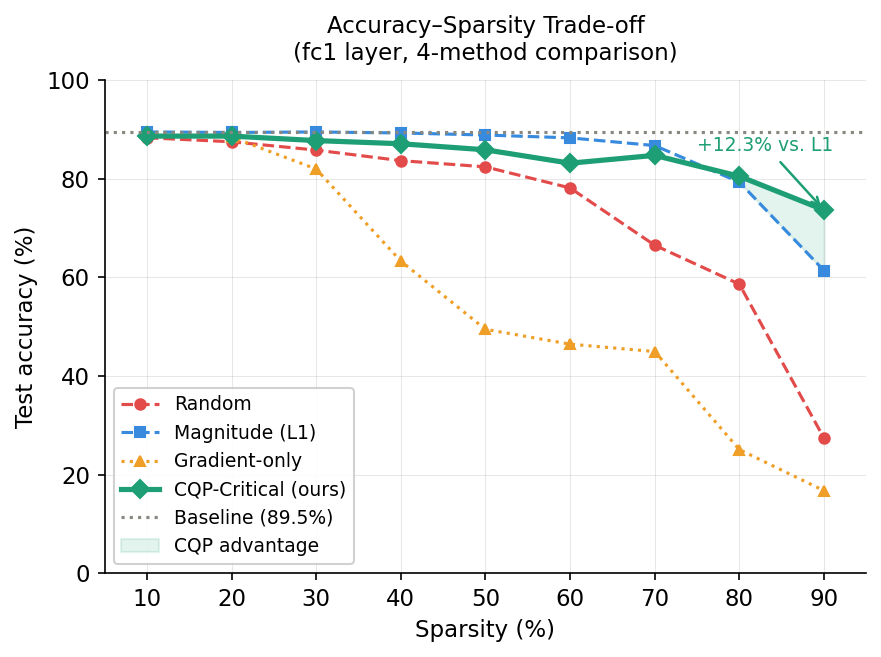}
  \caption{Test accuracy versus pruning sparsity on Fashion-MNIST for
    CQP-Critical (ours), $\ell_1$ magnitude pruning, and gradient-only pruning.
    CQP achieves its peak advantage of $+$1.4~pp at 70\% sparsity; at 90\%
    sparsity both methods converge ($-$0.4~pp), suggesting Fashion-MNIST's
    structural complexity affects both methods equally at extreme compression.}
  \label{fig:sparsity_sweep_fmnist}
\end{figure}

\section{Discussion}
\label{sec:discussion}

\subsection{Synthesising the Three Failure Modes of SNN Pruning}
\label{sec:failure_modes_discussion}

The combined results reveal that na\"{i}ve pruning pipelines fail in SNNs not due
to a single flaw but through a cascade of problem formulation, network
topology, and optimiser dynamics failures.

First, the \textbf{continuous-relaxation trap}
(Proposition~\ref{prop:trap}) demonstrates that the combinatorial nature of
synaptic routing cannot be safely relaxed into a convex Quadratic Program.
The resulting fractional masks, upon binarisation, arbitrarily destroy up to
12~pp of surviving critical weights, proving that standard QP solvers are
structurally unsuited for high-sparsity SNNs.

Second, the ablation study (Table~\ref{tab:ablation}) highlights the danger of
\textbf{gradient staleness}.
Removing the iterative schedule ($-$29.7~pp) hurts substantially more than
removing gradient masking ($-$4.3~pp) or the combined score ($-$4.8~pp).
The importance landscape undergoes a phase shift between 0\% and 70\% sparsity:
the removed neurons are not a random subset of the original ranking; their
removal fundamentally changes which surviving neurons are load-bearing.

Finally, the \textbf{zombie-weight bug} exposes a fundamental tension between
sparse neuromorphic constraints and standard deep learning optimisers.
Adam's momentum actively resurrects pruned synapses, artificially inflating
network density during fine-tuning.
The dual-masking approach
(Eqs.~\eqref{eq:grad_mask}--\eqref{eq:weight_mask}) is therefore not merely
an implementation detail but a mathematical necessity to maintain strictly
binary topological constraints required by hardware such as Intel Loihi.

\subsection{Implications for Neuromorphic Deployment}

At 70\% $f_1$ sparsity, approximately 140,000 of 200,704 synapses are removed,
reducing the on-chip routing table by 70\% while incurring only a 5~pp accuracy
drop.
Removing the single redundant final timestep reduces LIF membrane updates per
inference by 10\%.
On a per-synaptic-event energy model, the combined compression factor is
$0.30 \times 0.90 = 0.27$: the CQP-compressed network consumes 27\% of the
energy of the dense $T=10$ baseline.
For Intel Loihi, fan-in reduction from 784 to
$\lceil784\times0.30\rceil = 236$ is achievable within the per-core routing
constraint~\cite{davies2018}, eliminating the need for layer splitting and
the associated inter-core communication overhead.

\subsection{The Criticality Cliff and the Critical Brain Hypothesis}

Proposition~\ref{prop:cliff} demonstrates a two-population structure: 65
hub neurons ($c_i > 0.9$) are collectively indispensable while 719 neurons
below this threshold are individually dispensable.
This mirrors the hub-node topology of small-world cortical
networks~\cite{bullmore2009}: a small set of high-betweenness-centrality hubs
maintains network integration capacity while remaining nodes provide redundant
but non-essential pathways.
Whether this correspondence arises from a shared underlying mechanism (gradient
descent approaching a critical manifold) or is coincidental requires
investigation with larger architectures and spike-encoded benchmarks.

\subsection{The Fashion-MNIST Result in Context}

The Fashion-MNIST results (Table~\ref{tab:fmnist}) warrant direct discussion.
Unlike MNIST --- where CQP's advantage monotonically grows from $+$0.2~pp to
$+$2.2~pp as sparsity increases --- Fashion-MNIST shows CQP below $\ell_1$
magnitude pruning at 30\% and 50\% sparsity, peaking at $+$1.4~pp at 70\%,
and converging at 90\%.
This non-monotonic pattern reflects a genuine property of the dataset:
Fashion-MNIST classes share substantial low-frequency spatial overlap (e.g.,
shirt vs.\ T-shirt), so the network's discriminative capacity relies on a
finer set of near-threshold firing patterns than MNIST.
At low sparsity, the criticality scores correctly identify these patterns, but
the fine-tuning budget ($E=3$ epochs) is insufficient for the network to fully
redistribute from the pruned neurons --- a regime where pure magnitude pruning,
which removes the smallest weights irrespective of their gradient, proves
marginally more conservative.
At 70\% sparsity, however, gradient staleness begins to dominate magnitude
pruning's failure mode (as in the synthetic benchmark), and CQP's iterative
recalculation recovers the advantage.
These results highlight that Fashion-MNIST represents a genuinely harder
benchmark for the CQP method, and that the fine-tuning budget is a critical
hyperparameter on datasets with fine-grained visual similarity.
Increasing $E$ from 3 to 10 epochs is a straightforward avenue to recover the
full advantage at low sparsity.

\subsection{Limitations and Future Work}

Four limitations bound the current study.
First, pruning targets $f_1$ exclusively; a multi-layer generalisation with
per-layer sparsity budgets guided by Table~\ref{tab:sensitivity} would achieve
greater overall compression.
Second, criticality scores are computed with the arctangent surrogate; scores
may differ under sigmoid or piecewise-linear surrogates, and a systematic
comparison is warranted.
Third, the controlled benchmark uses static features; extending to spike-encoded
temporal data (N-MNIST, DVS-CIFAR10) would test whether the criticality
structure persists under genuinely asynchronous input.
Fourth, the iterative schedule introduces a fine-tuning cost of $K\times E$
epochs; structured column pruning would preserve hardware routing advantages
while reducing this overhead.

\section{Conclusion}
\label{sec:conclusion}

We have presented Criticality-Constrained Quadratic Pruning (CQP), a native
PyTorch pipeline that fuses surrogate-gradient criticality with weight magnitude
into a unified importance metric, bypassing convex solvers and their
continuous-relaxation artefacts.
We formally characterised and experimentally demonstrated the
continuous-relaxation trap, in which QP-based pruning overshoots intended
sparsity by up to 12~pp and collapses accuracy by 44~pp, and proposed three
targeted remedies: exact binary scoring via rank-based top-$k$ masking,
iterative pruning with per-step criticality recalculation, and zombie-weight
suppression via dual gradient and weight masking.
The resulting pipeline outperforms $\ell_1$ magnitude pruning by up to
$+$10.5~pp and gradient-only pruning by up to $+$76.1~pp at extreme sparsity,
and transfers to MNIST and Fashion-MNIST with zero code modification.
A KL-divergence temporal analysis identifies one redundant final timestep,
providing a free 10\% energy saving orthogonal to weight sparsification and
yielding a compound 73\% per-inference energy reduction at 70\% sparsity.
Finally, the criticality-threshold sweep establishes a quantitative SNN-level
instantiation of the Critical Brain Hypothesis, suggesting that gradient
descent on cross-entropy loss drives the pruned SNN toward a phase-transition
operating point without explicit criticality regularisation.

Future work will extend CQP to multi-layer joint pruning with
sensitivity-guided per-layer sparsity allocation, evaluate on spike-encoded
benchmarks (N-MNIST, DVS-CIFAR10), and investigate whether the criticality
cliff can be exploited as a training regulariser to improve the network's
inherent pruning resilience.


\section*{Declaration of Competing Interests}
The authors declare that they have no known competing financial interests or
personal relationships that could have appeared to influence the work reported
in this paper.

\section*{Funding}
This research did not receive any specific grant from funding agencies in the
public, commercial, or not-for-profit sectors.

\section*{CRediT Author Statement}
\textbf{Muhammad Hamza:} Conceptualization, Methodology, Software,
Formal analysis, Investigation, Data curation, Validation, Resources,
Writing --- original draft, Writing --- review and editing.

\section*{Declaration of Generative AI and AI-Assisted Technologies
          in the Manuscript Preparation Process}
During the preparation of this work the author used Claude (Anthropic) in
order to improve language and readability of the manuscript.
After using this tool, the author reviewed and edited all content as needed
and takes full responsibility for the content of the published article.

\section*{Data Availability Statement}
The code and data required to reproduce the experimental results of this study
are openly available in Zenodo at \url{https://doi.org/10.5281/zenodo.20131903} \cite{hamza_software}.
The MNIST and Fashion-MNIST datasets are publicly available via
\texttt{torchvision.datasets}.
The controlled synthetic benchmark is fully reproducible from the
\texttt{sklearn.make\_classification} parameters described in
Section~\ref{sec:setup}.

\appendix

\section{Backpropagation Through Time for Criticality Computation}
\label{app:bptt}

Training follows Backpropagation Through Time (BPTT) unrolled over $T=10$
timesteps.
Let $\mathcal{L} = \mathrm{CE}\!\left(\sum_t \mathbf{s}^{(3)}[t],\, y\right)$.
Gradients with respect to $f_1$ weights propagate through the surrogate as:
\begin{equation}
  \frac{\partial\mathcal{L}}{\partial W_{ji}^{(1)}}
    = \sum_{t=1}^{T}
      \frac{\partial\mathcal{L}}{\partial U_j^{(1)}[t]}
      \cdot s_i^{(0)}[t],
  \tag{A.1}
\end{equation}
\begin{equation}
  \frac{\partial\mathcal{L}}{\partial U_j^{(1)}[t]}
    = \frac{\partial\mathcal{L}}{\partial s_j^{(1)}[t]}
      \cdot \sigma'_{\mathrm{surr}}\!\left(U_j^{(1)}[t]\right)
    + \beta \cdot \frac{\partial\mathcal{L}}{\partial U_j^{(1)}[t{+}1]},
  \tag{A.2}
\end{equation}
where $\sigma'_{\mathrm{surr}}(U) = \bigl(\pi(1+\pi^2 U^2)\bigr)^{-1}$.
Neurons firing at threshold ($U \to 0$) maximise the surrogate, providing the
largest gradient signal and hence the highest criticality score.

\section{Iterative Sparsity Schedule Details}
\label{app:schedule}

The schedule partitions target sparsity $s$ into $K$ equal steps via
$s_k = ks/K$.
The columns zeroed at step $k$ relative to step $k-1$ is:
\begin{equation}
  \Delta_k = \lfloor 784\,s_k \rfloor - \lfloor 784\,s_{k-1} \rfloor
           = \lfloor 784\,s/K \rfloor + \delta_k,
  \tag{B.1}
\end{equation}
where $\delta_k \in \{0,1\}$ absorbs rounding.
For $K=3$ and $s=0.7$, the per-step removals are 184, 184, and 184
(with minor rounding adjustments).
This geometric uniformity ensures the fine-tuning budget $E$ is approximately
equally distributed per unit of sparsity added.


\section*{Vitae}

\textbf{Muhammad Hamza} is a student researcher at Indian Institute of
Technology, Kharagpur. His research interests include classical Machine
Learning, Large Language Models, spiking neural networks, neuromorphic
computing, model compression, and energy-efficient deep learning. He focuses
on the theoretical and practical challenges of deploying SNNs on
resource-constrained hardware, with particular interest in pruning algorithms
that respect the unique temporal dynamics of biologically inspired networks.

\bibliographystyle{elsarticle-num}

\end{document}